\documentclass[11pt]{article}
\usepackage{amssymb}
\usepackage{amsmath}
\usepackage{amsfonts}
\usepackage{graphicx}
\usepackage{color}
\usepackage[round]{natbib}
\newtheorem{defn}{\noindent \bf{Definition}}
\newtheorem{thm} {\noindent \bf{Theorem}}

\newtheorem{as}{\noindent \bf{Assumption}}
\newtheorem{lem}{\noindent \bf{Lemma}}

\newenvironment{proof}[1][Proof]{\noindent \textbf{#1.} }{\ \rule{0.5em}{0.5em}}
\newcounter{remno} 
\begin{document} 
\title{Adaptive Policies for Sequential Sampling under Incomplete
  Information and a Cost Constraint}

\author{Apostolos Burnetas and Odysseas Kanavetas\\
Department of Mathematics, University of Athens\\
Panepistemiopolis, Athens 15784, Greece\\
\{aburnetas,okanav\}@math.uoa.gr}

\maketitle

\begin{abstract}
  We consider the problem of sequential sampling from a finite number
  of independent statistical populations to maximize the expected
  infinite horizon average outcome per period, under a constraint that
  the expected average sampling cost does not exceed an upper
  bound. The outcome distributions are not known. We construct a class
  of consistent adaptive policies, under which the average outcome
  converges with probability 1 to the true value under complete information for all
  distributions with finite means. We also compare the rate of
  convergence for various policies in this class using simulation.
\end{abstract}

\section{Introduction}
In this paper we consider the problem of sequential sampling from $k$
independent statistical populations with unknown distributions. The
objective is to maximize the expected outcome per period achieved over
infinite horizon, under a constraint that the expected sampling cost
per period does not exceed an upper bound.  The introduction of a
sampling cost introduces a new dimension in the standard tradeoff
between experimentation and profit maximization faced in problems of
control under incomplete information. The sampling cost may prohibit
using populations with high mean outcomes because their sampling cost
may be too high. Instead, the decision maker must identify the subset
of populations with the best combination of outcome versus cost and
allocate the sampling effort among them in an optimal manner.

From the mathematical point of view, this class of problems
incorporates statistical methodologies into mathematical programming
problems. Indeed, under complete information, the problem of effort
allocation under cost constraints is typically formulated in terms
of linear or nonlinear programming. However when some of the problem
parameters are not known in advance but must be estimated by
experimentation, the decision maker must design adaptive learning and
control policies that ensure learning about the parameters while at
the same time ensuring that the profit sacrificed for the learning
process is as low as possible.

The model in this paper falls in the general area of multi-armed
bandit problems, which was initiated by \cite{rob52}, who proposed a
simple adaptive policy for sequentially sampling from two unknown
populations in order to maximize the expected outcome per unit time
infinite horizon.  \cite{lai-rob85} generalize the results by
constructing asymptotically efficient adaptive policies with optimal
convergence rate of the average outcome to the optimal value under
complete information and show that the finite horizon loss due to
incomplete information increases with logarithmic rate.
\cite{kat-rob95} prove that simpler index-based efficient policies
exist in the case of normal distributions with unknown means, while
\cite{ad-mab} extend the results on efficient policies in the
nonparametric case of discrete distributions with known support.

In a finite horizon \cite{kul00} develop a minimax version of the
\cite{lai-rob85} results for two populations, while \cite{Auer02}
construct policies which also achieve logarithmic regret uniformly
over time, rather than only asymptotically.

\vspace{0in} In all works mentioned above there is no side constraint
in sampling. Problems with adaptive sampling and side constraints are
scarce in the literature. \cite{Wang} considers a multi-armed bandit
model with constraints and adopts a Bayesian formulation and the
Gittins-index approach. The paper proposes several heuristic policies.
\cite{Pez} also consider the problem of estimating the distribution of
a single population with sampling cost under the assumption that the
number of users who will benefit from the depends on the outcome of
the estimation. Finally, \cite{mad04} present computational complexity
analysis for a version of the multi-armed bandit problem with Bernoulli
outcomes and Beta priors, where there is a total budget for
experimentation, which must be allocated to sampling from the
different populations.

Another approach, which is closer to the one we adopt here is to
consider the family of stochastic approximations and reinforcement
learning algorithms. The general idea is to select the sampled
population following a randomized policy with randomization
probabilities that are adaptively modified after observing the outcome
in each period. The adaptive scheme is based on the stochastic
approximation algorithm. Algorithm of this type are analyzed in
\cite{Poz} for the more general case where the population outcomes
have Markovian dynamics instead of being i.i.d..

The contribution of this paper is the construction of a family of
policies for which the average outcome per period converges to the
optimal value under complete information for all distributions of
individual populations with finite means. In this sense, it
generalizes the results of \cite{rob52} by including a sampling cost
constraint.  The paper organized as follows. In Section $2$, we
describe the model in the complete and incomplete information
framework. In Section $3$, we construct a class of adaptive sampling
policies and prove that it is consistent. In Section $4$, we explore
the rate of convergence of the proposed policies using
simulation. Section 5 concludes.

\bigskip

\section{Model description}

Consider the following problem in adaptive sampling. There are $k$
independent statistical populations, $i=1,\ldots,k$. Successive
samples from population $i$ constitute a sequence of i.i.d. random
variables $X_{i1},X_{i2},\ldots$ following a univariate distribution
with density $f_{i}(\cdot)$ with respect to a nondegenerate measure
$v$. Then the stochastic model is uniquely determined by the vector
$f=(f_1,\ldots,f_k)$ of individual pdf's. Given f let
$\underline{\mu}(f)$ be the vector of expected values, i.e. $\mu_i
(f)=E^{f_i}(X_i)$. The form of $f$ is not known. In each period the
experimenter must select a population to obtain a single sample
from. Sampling from population $i$ incurs cost $c_i$ per sample and
without loss of generality we assume $c_1 \leq c_2 \leq \ldots \leq
c_k$, but not all equal. The objective is to maximize the expected
average reward per period subject to the constraint that the expected
average sampling cost per period over infinite horizon does not exceed
a given upper bound $C_0$. Without loss of generality we assume
$c_1\leq C_0<c_k$. Indeed if $C_0<c_1$ then the problem is
infeasible. On the other hand if $C_0\geq c_k$ then the cost
constraint is redundant. Let $d=max\{ j:c_j\leq C_0 \}$. Then $1\leq
d< k$ and $c_d\leq C_0 <c_{d+1}$.

\subsection{Complete information framework}

\vspace{0in}We first analyze the complete information problem. If all
$f_{i}(\cdot)$ are known, then the problem can be modeled via linear
programming. Consider a randomized sampling policy which at each
period selects population $j$ with probability $x_j$, for
$j=1,\ldots,k$. To find a policy that maximizes the expected reward,
we can formulate the following linear program in standard form

\begin{eqnarray}
z^{*} & = & \max
\sum_{j=1}^{k}\mu_{j} x_j \nonumber\\
&&\sum_{j=1}^{k}c_j x_j + y = C_0 \label{lp}\\
&&\sum_{j=1}^{k}x_j =1 \nonumber\\
&&x_j \geq 0,\forall j. \nonumber
\end{eqnarray}

\vspace{0in}\noindent Note that $z^*$ depends on $f$ only through the
vector $\underline{\mu}(f)$, i.e. $z^*$ is the same for all collections
of pdf with the same $\mu$. Therefore in the remainder we will denote
$z^*$ as a function of the unknown mean vector $\underline{\mu}$.

In the analysis we will also use the dual linear program (DLP) of
\eqref{lp},
\begin{eqnarray*}
z^{*}_{D} & = & \min \ g+C_0 \lambda \\ &&
g+c_1 \lambda \geq \mu_1 \\ &&\hspace{1cm}
\vdots \label{dlp}\\ 
&& g+c_k \lambda \geq \mu_k \\\ && g\in
\mathbb{R},\lambda \geq 0 ,
\end{eqnarray*}
\noindent with two variables $\lambda$ and $g$ which
correspond to the first and second constraints of \eqref{lp},
respectively.

\noindent
The basic matrix $B$ corresponding to a   Basic Feasible Solution (BFS) of
problem \eqref{lp} may take one of two forms:

\vspace{0in}\noindent In the first case, the basic variables are $x_i,
x_j$, for two populations $i,j$, with $c_i \leq C_0 \leq c_j, c_i <
c_j$, and the basic matrix is
\begin{equation} \mathbf{B} = \left( \begin{array}{cc} c_i & c_j\\ 1 &
1\\
\end{array} \right). \nonumber
\end{equation}

The BFS is then
\begin{equation} x_{i}=\frac{c_j - C_0}{c_i -c_j}, \ x_{j}=\frac{C_0
-c_i}{c_i -c_j}, \ \text{and} \ x_m =0 \ \text{for} \ m\neq
i,j, \ y=0,\nonumber
\end{equation}
\vspace{0in}\noindent with
\begin{equation}
z(\underline{x})=\mu_{i}x_{i}+\mu_{j}x_{j}.
\nonumber
\end{equation}
The solution is nondegenerate when $c_i < C_0 < c_j$ and degenerate
when $C_0 = c_i$ or $C_0 = c_j$. In the latter case, it corresponds to
sampling from a single population $l=i$ or $l=j$, respectively:
\begin{equation} x_l =1, \ \ x_m =0 \ \forall m\neq l, \ \ y=0,\nonumber
\end{equation}
\noindent with
\begin{equation} z(\underline{x})=\mu_{l}.  \nonumber
\end{equation}

The second case of a BFS corresponds to basic variables $x_i, y$ for a
population $i$ with $c_i \leq C_0$. The basic matrix is
\begin{equation} \mathbf{B}=\left( \begin{array}{cc} c_i & 1\\ 1 & 0
\\
\end{array} \right).\nonumber
\end{equation}
In this case the BFS corresponds to sampling from population $i$ only
\begin{equation}
x_i =1, \ \ x_m =0 \ \forall m\neq i, \ \ y=C_0 -c_i,\nonumber
\end{equation}
\noindent with
\begin{equation}
z(\underline{x})=\mu_{i}.  \nonumber
\end{equation}
The solution is nondegenerate if $c_i < C_0$, otherwise it is degenerate.

From the above it follows that a BFS is degenerate if $x_l = 1$ for
some $l$ with $c_l=C_0$. Any basic matrix $B$ that includes $x_l$ as a
basic variable corresponds to this BFS.

For a BFS $x$ let
$$
b = \{ i : x_i >0\}.
$$
Then, either $b=\{i,j\}$ for some $i,j$ with
$i\leq d\leq j$, or $b=\{i\}$ for some $i\leq d$. There is a one to
one correspondence between basic feasible solutions and sets $b$ of
this form.  We use $K$ to denote the set of BFS, or equivalently
\begin{equation}
  K=\{ b \ : \ b=\{i,j\}, \ i\leq d\leq j \text{ or }b=\{i\}, \ i\leq d  \}.\nonumber
\end{equation}
\noindent Since the feasible region of \eqref{lp} is bounded, $K$ is finite.

For a basic matrix $B$, let $v^B= (\lambda^B, g^B)$ denote the dual
vector corresponding to $B$, i.e., $v^B = \mu_B B^{-1}$, where
$\mu_B=(\mu_i, \mu_j)$, or $\mu_B=(\mu_i, 0)$, depending on the form
of $B$.

Regarding optimality, a BFS is optimal if and only if for at least one
corresponding basic matrix $B$ the reduced costs (dual slacks) are all
nonnegative:
\begin{equation}
  \label{eq:phidef}
\phi^B_a \equiv c_{\alpha} \lambda^B + g^B - \mu_{\alpha}
  \geq 0, \ \alpha=1,\ldots, k.\nonumber
\end{equation}

A basic matrix $B$ satisfying this condition is optimal. Note that if
an optimal BFS is degenerate, then not all basic matrices
corresponding to it are necessarily optimal.

It is easy to show that the reduced costs can be expressed as a linear
combinations
$\phi_{\alpha}^{B}=\underline{w}_{\alpha}^{B}\underline{\mu}$, where
$\underline{w}_{\alpha}^{B}$ is an appropriately defined vector that
does not depend on $\underline{\mu}$.

We finally define the set with optimal solutions of \eqref{lp} for a
$\underline{\mu}$,
\begin{equation}
s(\underline{\mu})=\{ b \in K : b \text{ corresponds to an optimal BFS} \}.
\nonumber
\end{equation}

An optimal solution of \eqref{lp} specifies randomization
probabilities that guarantee maximization of the average reward
subject to the cost constraint. Note that an alternative way to
implement the optimal solution, without randomization, is to sample
periodically from all populations so that the proportion of samples
from each population $j$ is equal to $x_j$. This characterization of a
policy is valid if randomization probabilities are rational.

\subsection{Incomplete information framework}

In this paper we assume that the population distributions are
unknown. Specifically we make the following assumption.

\begin{as} The outcome distributions are independent, and the expected
values $\mu_{\alpha}=E(X_{\alpha})<\infty$, $\alpha=1,\ldots,k$.
\end{as}

Let $F$ be the set of all $f=(f_1,\ldots,f_k)$ which satisfy
A.1. Class $F$ is the effective parameter set in the incomplete
information framework. Under incomplete information, a policy as that
in Section 2.1, which depends on the actual value of $\underline{\mu}$, is not
admissible. Instead we restrict our attention to the class of adaptive
policies, which depend only on the past observations of selections and
outcomes.

Specifically, let $A_t,X_t$ , $t=1,2,...$ denote the population
selected and the observed outcome at period $t$. Let $h_t =(\alpha_1
,x_1,....,\alpha_{t-1},x_{t-1})$ be the history of actions and
observations available at period t.

\vspace{0in}An adaptive policy is defined as a sequence $\pi
=(\pi_1,\pi_2,...)$ of history dependent probability distributions on
$\{1,...,k\}$, such that

\begin{equation} \pi_t (j,h_t)=P(A_t =j|h_t).\nonumber
\end{equation}

\vspace{0in}Given the history $h_n$, let $T_n (\alpha)$ denote the
number of times population $\alpha$ has been sampled during the first
n periods

\begin{equation} T_n (\alpha)= \sum_{t=1}^{n}1\{A_t
=\alpha\}.\nonumber
\end{equation}

Let $S_{n}^{\pi}$ be the reward up to period $n$:

\begin{equation} S_{n}^{\pi}=\sum_{t=1}^{n}X_t,\nonumber
\end{equation} and $C_{n}^{\pi}$ be the total cost up to period $n$:

\begin{equation} C_{n}^{\pi}=\sum_{t=1}^{n}c_{A_t}.\nonumber
\end{equation}

These quantities can be used to define the desirable properties of an
adaptive policy, namely feasibility and consistency.

\begin{defn} A policy $\pi$ is called feasible if

\begin{equation}
\label{eq:feasible}
 \limsup_{n \rightarrow \infty} \frac{E^{\pi}(C_{n}^{\pi})}{n} \leq C_0, \ \forall f \in F.
\end{equation}

\end{defn}

\begin{defn} A policy $\pi$ is called consistent if it is feasible and

\begin{equation} \lim_{n \rightarrow \infty} \frac{S_{n}^{\pi}}{n}
=z^{*}(\underline{\mu}),\text{ a.s. } \ \forall f \in F. \nonumber
\end{equation}

\end{defn}

Let $\Pi^F$ and $\Pi^C$ denote the class of feasible and consistent
policies, respectively. The above properties are reasonable
requirements for an adaptive policy. The first ensures that the
long-run average sampling cost does not exceed the budget. The second
definition means that the long-run average outcome per period achieved
by $\pi$ converges with probability one to the optimal expected value
that could be achieved under full information, for all possible
population distributions satisfying A.1.

Note that consistency as defined in Definition $2$ is equivalent to
the notion of strong consistency of an estimator function.

\section{Construction of a consistent policy}

A key question in the incomplete information framework is whether
feasible and, more importantly, consistent policies exist and how they
can be constructed.

It is very easy to show that feasible policies exist, since the
sampling costs are known. Indeed any randomized policy, such as those
defined in Section 2.1, with randomization probabilities satisfying
the constraints of LP \eqref{lp} is feasible for any distribution
$f$. Thus, $\Pi^F\neq\emptyset$.

On the other hand, the construction of consistent policies is not
trivial. A consistent policy must accomplish three goals: First to be
feasible, second to be able to estimate the mean outcomes from all
populations, and third, in the long-run, to sample from the nonoptimal
populations rarely enough so as not to affect the average profit.

In this section we establish the existence of a class of
consistent policies. The construction follows the main idea of
\cite{rob52}, based on sparse sequences, which is adapted to ensure
feasibility.

We start with some definitions. For any population $j$, let
$\hat{\mu}_{j,t}$, $t=1,2,\ldots$ be a strongly consistent estimator
of $\mu_j$, i.e. $\lim_{t\rightarrow\infty}\hat{\mu}_{j,t}=\mu_j$
a.s.-$f_j$. Such estimators exist; for example from Assumption 1, the
sample mean $\overline{X}_{j,t}=\frac{1}{t}\sum_{k=1}^{t}X_{j,k}$ is
strongly consistent.

For any $n$, let $\hat{\underline{\mu}}_n=(\hat{\mu}_{j,T_{j}(n)}, \
j=1,\ldots,k)$ be the vector estimates of $\underline{\mu}$ based on
the history up to period $n$. Also let $\hat{z}_n=z(\hat{\underline{\mu}}_n)$ denote
the optimal value of the linear program in \eqref{lp} where the estimates
are used in place of the unknown mean vector in the
objective. $\hat{z}_n$ will be referred to as the
Certainty-Equivalence LP. Note that $s(\hat{\underline{\mu}}_n)$ is the set of optimal BFS of $\hat{z}_n$.

The solution of $\hat{z}_n$ corresponds to a sampling policy
determined by an optimal vector $\hat{x}_n$, so that
$\hat{z}_n=\hat{\underline{\mu}}_{n}^{'}\hat{x}_n$.

We next define a class of sampling policies, which we will show to be
consistent. Consider $k$ nonoverlapping sparse sequences of positive
integers,

\begin{equation} \tau_j=\{ \tau_{j,m}, \ m=1,2,\ldots \}, \
j=1,\ldots,k,
\nonumber
\end{equation}
such that
\begin{equation}
  \label{eq:sparse}
\lim_{m\rightarrow\infty}\frac{\tau_{j,m}}{m}=\infty, \ j=1,\ldots,k.
\end{equation}

Now define policy $\pi^0$ which in period $n$ selects any population $j$
with probability equal to

\begin{equation} \pi^0(j|h_n)=\left\{\begin{array}{ll}1 &, \text{ if }
\tau_{j,m}=n \text{ for some }m\geq1 \\ \hat{x}_{n,j} &, \text{
otherwise}\end{array}\right.\nonumber
\end{equation}

\noindent where $\hat{x}_{n}$ is any optimal BFS of the
certainty-equivalence LP $\hat{z}_n$.

The main idea in $\pi^0$ is that at periods which coincide with the
terms of sequence $\tau_j$, population $j$ is selected regardless of
the history. These instances are referred to as forced selections of
population $j$. The purpose of forced selections is to ensure that all
populations are sampled infinitely often, so that the estimate vector
$\hat{\underline{\mu}}_n$ converges to the true mean $\underline{\mu}$
as $n\rightarrow\infty$.

On the other hand, because sequences $\tau_j$ are sparse, the fraction
of forced selections periods converges to zero for all $j$, so that
sampling from the nonoptimal populations does not affect the average
outcome in the long-run.

In the remaining time periods, which do not coincide with a sparse
sequence term, the sampling policy is that suggested by the certainty
equivalence LP, i.e., the experimenter in general randomizes between
those populations, which, based on the observed history, appear to be
optimal.

In the next theorem we prove the main result of the paper, namely that
$\pi^0\in\Pi^C$. The proof adapts the main idea of 
\cite{rob52} to the problem with the cost constraint.

\begin{thm}
\label{consthm}
Policy $\pi^0$ is consistent.
\end{thm}

Before we show Theorem $1$, we prove an intermediate result which
shows that if in some period the certainty equivalence LP yields an
optimal solution that is non-optimal under the true distribution $f$,
then the estimate of at least one population mean must be sufficiently
different from the true value. We use the supremum norm $\| x \|=
\max_j |x|$.

\begin{lem}
  For any $\underline{\mu}$ there exists $\epsilon>0$ such that for
  any $n=1,2,\ldots$ if $b\in s(\hat{\underline{\mu}}_n)$ and $b\notin
  s(\underline{\mu})$ for some $b\in K$, then $\|
  \underline{\mu}-\hat{\underline{\mu}}_n \|\geq\epsilon$.
\end{lem}

\begin{proof}
  Since $b\notin s(\underline{\mu})$, we have that for any basic
  matrix $B'$ corresponding to BFS $b$ there exists at least one
  $m\in\{ 1,\ldots,k \}$ such that $\phi_{m}^{B'}(\underline{\mu})<0$.
  Therefore,
\begin{equation}
-\underline{w}_{m}^{B'}\underline{\mu}=-\phi_{m}^{B'}(\underline{\mu})>0.
\label{lem1}
\end{equation}

In addition, since $b\in s(\hat{\underline{\mu}}_n)$, there exists a
basic matrix $B$ corresponding to $b$, such for any $m\in\{ 1,\ldots,k
\}$ it is true that $\phi_{m}^{B}(\hat{\underline{\mu}}_n)\geq0$, thus,
\begin{equation}
  \underline{w}_{m}^{B}\hat{\underline{\mu}}_n=\phi_{m}^{B}(\hat{\underline{\mu}}_n)\geq 0.\label{lem2}
\end{equation}

For this basic matrix $B$, it follows from \eqref{lem1} and \eqref{lem2} that
\begin{equation}
\underline{w}_{m}^{B}\hat{\underline{\mu}}_n-\underline{w}_{m}^{B}\underline{\mu}\geq-\phi_{m}^{B}(\underline{\mu}) =|\phi_{m}^{B}(\underline{\mu})| >0\nonumber
\end{equation}

\begin{equation}
\Rightarrow \underline{w}_{m}^{B}(\hat{\underline{\mu}}_n-\underline{\mu})\geq|\phi_{m}^{B}(\underline{\mu})|\nonumber
\end{equation}

\begin{equation}
\Rightarrow k\|\underline{w}_{m}^{B}\|\|\hat{\underline{\mu}}_n-\underline{\mu}\|\geq|\phi_{m}^{B}(\underline{\mu})|\nonumber
\end{equation}

\begin{equation}
\Rightarrow \|\hat{\underline{\mu}}_n-\underline{\mu}\|\geq\frac{|\phi_{m}^{B}(\underline{\mu})|}{k\|\underline{w}_{m}^{B}\|},\nonumber
\end{equation}
\noindent because from the property $\underline{w}_{m}^{B}\underline{\mu}<0$ it follows that $\| \underline{w}_{m}^{B} \|>0$.

Now let
\begin{equation}
\epsilon=\min_{b\in K,b\notin s(\underline{\mu})}  \min_{B \in b} \min_{m\in\{1,\ldots,k\}}\left\{\frac{|\phi_{m}^{B}(\underline{\mu})|}{k\|\underline{w}_{m}^{B}\|}: \ \phi_{m}^{B}(\underline{\mu})<0\right\}>0. \nonumber
\end{equation}
where the minimization over $B \in b$ is taken over all basic matrices corresponding to BFS $b$.

Then $\|\hat{\underline{\mu}}_n-\underline{\mu}\|\geq\epsilon$.

\end{proof}

\noindent {\bf Proof of Theorem 1.}



For $i=1,\ldots,k$ let

\begin{equation}
SS_i (n)=\sum_{t=1}^{n}1\{ \tau_{i,m}=t,\text{ for some }m \},\nonumber
\end{equation}

\noindent denote the number of periods in $\{1,\ldots,n\}$ where a forced selection from population $i$ is performed.

Also let,

\begin{eqnarray}
Y^{b}_{j}(n)&=&\sum_{t=1}^{n}1\{b\in s(\hat{\underline{\mu}}_t), b \text{ is used in period $t$, and $j$ is sampled from,}\nonumber\\
&& \text{due to randomization in }b\}.\nonumber\\
Y^b (n)&=&\sum_{j\in b}Y_{j}^{b}(n),\nonumber\\
Y(n)&=&\sum_{b\in s(\underline{\mu})}Y^b (n).\nonumber
\end{eqnarray}

Since these include all possibilities of selection in a period, it is true that

\begin{equation}
n=\sum_{i=1}^{k}SS_i (n)+\sum_{b\notin s(\underline{\mu})}Y^{b}(n)+\sum_{b\in s(\underline{\mu})}Y^{b}(n).\nonumber
\end{equation}

Now let $W_n$ denote the sum of outcomes in periods where true optimal BFS are used:

\begin{equation}
W_n=\sum_{b\in s(\underline{\mu})}\sum_{t=1}^{n} X_t \cdot 1\{ b\text{ is used in period t} \}.\nonumber
\end{equation}

To show the theorem we will prove that

\begin{eqnarray}
&&\lim_{n\rightarrow\infty}\frac{SS_i (n)}{n}= 0, \ i=1,\ldots,k \label{thm1*}\\
&&\lim_{n\rightarrow\infty}\sum_{b\notin s(\underline{\mu})}\frac{Y^{b}(n)}{n}= 0, \text{ a.s.} ,\label{thm2*}\\
&&\lim_{n\rightarrow\infty}\frac{W_n}{n}=z^*(\underline{\mu}), \text{ a.s.}\label{thm3*}.
\end{eqnarray}

First, \eqref{thm1*} holds since $\tau_{i,m}$ are sparse for all $i$. To show \eqref{thm2*}, in no forced selection periods, in order to sample from a BFS $b$ it is necessary but not sufficient that $b\in s(\hat{\underline{\mu}}_n)$, thus

\begin{equation}
Y^{b}(n)\leq \sum_{t=1}^{n}1\{ b \in s(\hat{\underline{\mu}}_t) \}.\nonumber
\end{equation}

For any $b \in s(\hat{\underline{\mu}}_t)$ and $b\notin s(\underline{\mu})$, it follows from Lemma $1$ that

\begin{equation}
\|\hat{\underline{\mu}}_t - \underline{\mu}\|
\geq\epsilon.\nonumber
\end{equation}

Therefore, for $b\notin s(\underline{\mu})$

\begin{eqnarray}
Y^{b}(n)&\leq& \sum_{t=1}^{n}1\{ b \in s(\hat{\underline{\mu}}_t) \}\nonumber\\
&\leq&\sum_{t=1}^{n}1\{ |\hat{\underline{\mu}}_t - \underline{\mu}\|
\geq\epsilon \}\nonumber
\end{eqnarray}

\noindent thus,

\begin{equation}
\frac{Y^b (n)}{n}\leq \frac{1}{n}\sum_{t=1}^{n}1\{\| \hat{\underline{\mu}}_t -
\underline{\mu} \| \geq\epsilon\}\rightarrow 0, \ n\rightarrow\infty,\text{ a.s.},\nonumber
\end{equation}

\noindent because $\hat{\underline{\mu}}_t\rightarrow\underline{\mu}$, a.s., since $\hat{\underline{\mu}}_t$ is strongly consistent estimator, thus \eqref{thm2*} holds.

Now to show \eqref{thm3*} we rewrite $W_n$ as

\begin{eqnarray}
\frac{W_n}{n}&=&\frac{1}{n}\sum_{b\in s(\underline{\mu})}\sum_{t=1}^{n}X_t\cdot1\{ b\text{ is used in period }t \}\nonumber\\
&=&\frac{1}{n}\sum_{b\in s(\underline{\mu})}\sum_{j\in b}\sum_{t=1}^{n}X_t\cdot1\{ b\text{ is used in period }t\text{ and j is sampled from} \}\nonumber\\
&=&\frac{1}{n}\sum_{b\in s(\underline{\mu})}\sum_{j\in b}Y_{j}^{b}(n)\cdot \overline{X}_{j,Y_{j}^{b}(n)}\nonumber\\
&=&\sum_{b\in s(\underline{\mu})}\frac{Y^{b}(n)}{n}\cdot\sum_{j\in b}\frac{Y_{j}^{b}(n)}{Y^b (n)}\cdot\overline{X}_{j,Y_{j}^{b}(n)}.\nonumber
\end{eqnarray}

From this expression it follows that

\begin{eqnarray}
\frac{W_n}{n} -z^*&=&\sum_{b\in s(\underline{\mu})}\frac{Y^{b}(n)}{n}\cdot\sum_{j\in b}\frac{Y_{j}^{b}(n)}{Y^b (n)}\cdot\overline{X}_{j,Y_{j}^{b}(n)}- z^*\nonumber\\
&=&\sum_{b\in s(\underline{\mu})}\frac{Y^{b}(n)}{n}\cdot z_{n}^{b} -z^*,\nonumber
\end{eqnarray}

\noindent where $z_{n}^{b}=\sum_{j\in b}\frac{Y_{j}^{b}(n)}{Y^b (n)}\cdot\overline{X}_{j,Y_{j}^{b}(n)}$.

Since $Y(n)=\sum_{b\in s(\underline{\mu})}Y^b (n)$, we have

\begin{eqnarray}
\frac{W_n}{n} -z^* &=& \sum_{b\in s(\underline{\mu})}\frac{Y^{b}(n)}{n}\cdot z_{n}^{b} -z^* +\frac{Y(n)}{n}z^*-\frac{Y(n)}{n}z^* \nonumber\\
&=&\sum_{b\in s(\underline{\mu})}\frac{Y^{b}(n)}{n}\cdot (z_{n}^{b} -z^*)-(1-\frac{Y(n)}{n})z^*.\nonumber
\end{eqnarray}

To show \eqref{thm3*} we will prove that

\begin{equation}
\frac{Y^{b}(n)}{n}\cdot (z_{n}^{b} -z^*)\rightarrow0\text{ a.s. }\forall b\in s(\underline{\mu}), \text{ and }\frac{Y(n)}{n}\rightarrow1,\text{ a.s.}.\nonumber
\end{equation}

Random variable $Y^b (n)$ is increasing in $n$ and $0\leq Y^b(n)\leq n$, thus either $Y^b (n)\rightarrow\infty$ or $Y^b (n)\rightarrow M$ for some $M<\infty$. We define the following events:

\begin{equation}
D=\{ Y^b (n)\rightarrow\infty \} \text{ and } D^c=\{ Y^b (n)\rightarrow M \}.\nonumber
\end{equation}

Now let $P(D)=p$ and $P(D^c)=1-p$. Also let

\begin{equation}
A=\{ \lim_{n\rightarrow\infty}\frac{Y^{b}(n)}{n}\cdot (z_{n}^{b} -z^*)=0 \}.\nonumber
\end{equation}

Then $P(A)=P(A|D)\cdot p+P(A|D^c)\cdot(1-p)$.

Now,

\begin{eqnarray}
P(A|D)&=&P(\lim_{n\rightarrow\infty}\frac{Y^{b}(n)}{n}\cdot (z_{n}^{b} -z^*)=0|\lim_{n\rightarrow\infty} Y^b (n)=\infty)\nonumber\\
& \geq & P(\lim_{n\rightarrow\infty}z_{n}^{b} -z^*=0|\lim_{n\rightarrow\infty} Y^b (n)=\infty)\nonumber\\
&=&1,\nonumber
\end{eqnarray}

\noindent from the strong law of large numbers, since $\frac{Y^b (n)}{n}\leq1$ $\forall$ $n$, and

\begin{equation}
P(A|D^c)=P(\lim_{n\rightarrow\infty}\frac{Y^{b}(n)}{n}\cdot (z_{n}^{b} -z^*)=0|\lim_{n\rightarrow\infty} Y^b (n)=M< \infty)=1,\nonumber
\end{equation}

\noindent since in this case $z_{n}^{b}-z^*$ is bounded for any finite $n$.

Therefore, $P(A)=1$, thus

\begin{equation}
\frac{Y^{b}(n)}{n}\cdot (z_{n}^{b} -z^*)\rightarrow0, \ n\rightarrow\infty, \text{ a.s. }, \ \forall b\in s(\underline{\mu}).\nonumber
\end{equation}

Finally,

\begin{equation}
\frac{Y(n)}{n}=\sum_{b\in s(\underline{\mu})}\frac{Y^{b}(n)}{n}=1-\sum_{t=1}^{n}\frac{SS_i (n)}{n}- \sum_{b\notin s(\underline{\mu})}\frac{Y^{b}(n)}{n}\rightarrow 1,\text{ a.s.}, \ n\rightarrow \infty \nonumber.
\end{equation}

Thus the proof of the theorem is complete.

$\blacksquare$

\section{Rate of Convergence - Simulations}

From the results of the previous section it follows that there exists
significant flexibility in the construction of a consistent sampling
policy.  Indeed, any collection of sparse sequences of forced
selection periods satisfying \eqref{eq:sparse} guarantees that Theorem
\ref{consthm} holds.

In this section we refine the notion of consistency and examine how
the rate of convergence of the average outcome to the optimal value is
affected by different types of sparse sequences.  Furthermore, since
the sensitivity analysis will be performed using simulation, it is
more appropriate to use the expected value of the deviation as the
convergence criterion. We thus consider the expected difference of the
average outcome under a consistent policy $\pi$ from the optimal
value:
$$
d_n^{\pi}( \underline{\mu})=E^{\pi}\left ( \frac{W_n}{n}\right
)-z^*(\underline{\mu}).
$$

Note that the almost sure convergence of $\frac{W_n}{n}$ to
$z^*(\underline{\mu})$ proved in Theorem \ref{consthm} does not imply
convergence in expectation, unless further technical assumptions on
the unknown distributions are made. For the purpose of our simulation
study, we will further assume that the outcomes of any population are
absolutely bounded with probability one, i.e., $P(|X_j|\leq u) = 1$,
for some $u>0$. Under this assumption it is easy to show that Theorem
1 implies
\begin{equation}
  \label{eq:convexp}
  \lim_{n \to \infty}d_n^{\pi}(\underline{\mu}) = 0, 
\end{equation}
for any consistent policy $\pi$ and any vector  $\underline{\mu}$.

To explore the rate of convergence in \eqref{eq:convexp}, we performed a
simulation study, for a problem with $k=4$ populations. The outcomes
of population $i$ follow binomial distribution with parameters $(N,
p_i)$, where $p_1=0.3, p_2=0.5, p_3=0.9, p_4=0.8$. The vector of
expected values is thus $\underline{\mu}=(1.5, 2.5, 4.5, 4)$. The cost
vector is $c=(3,4,8,10)$ and $C_0=5$. Under this set of values the
optimal policy under incomplete information is $x=(0, 3/4, 1/4, 0)$,
$y=0$ and $z^*(\underline{\mu})=3$, i.e., it is optimal to randomize
between populations 2 and 3, the expected sampling cost per period is
equal to 5 and the expected average reward per period is equal to 3.

For the above problem we simulated the performance of a consistent
policy for sparse sequences of power function form:
\[
\{\tau_{j,m} = \ell_j + m^b, m=1,2,\ldots, \}, j=1,\ldots, k,
\]
where $\ell_j$ are appropriately defined constants which ensure that
the sequences are not overlapping, and the exponent parameter $b$ is
common for all populations. We compared the convergence rate in
\eqref{eq:convexp} for five values of $b$: (1.2, 1.5, 2, 3, 5). For
each value of $b$ the corresponding policy was simulated for 1000
scenarios of length $n=10^4$ periods each, to obtain an estimate of
the expected average outcome per period
$d_n^{\pi}(\underline{\mu})$. The results of the simulations are
presented in Figure \ref{fig:av_var}.
\begin{figure}
  \centering
  \includegraphics[width=1\textwidth]{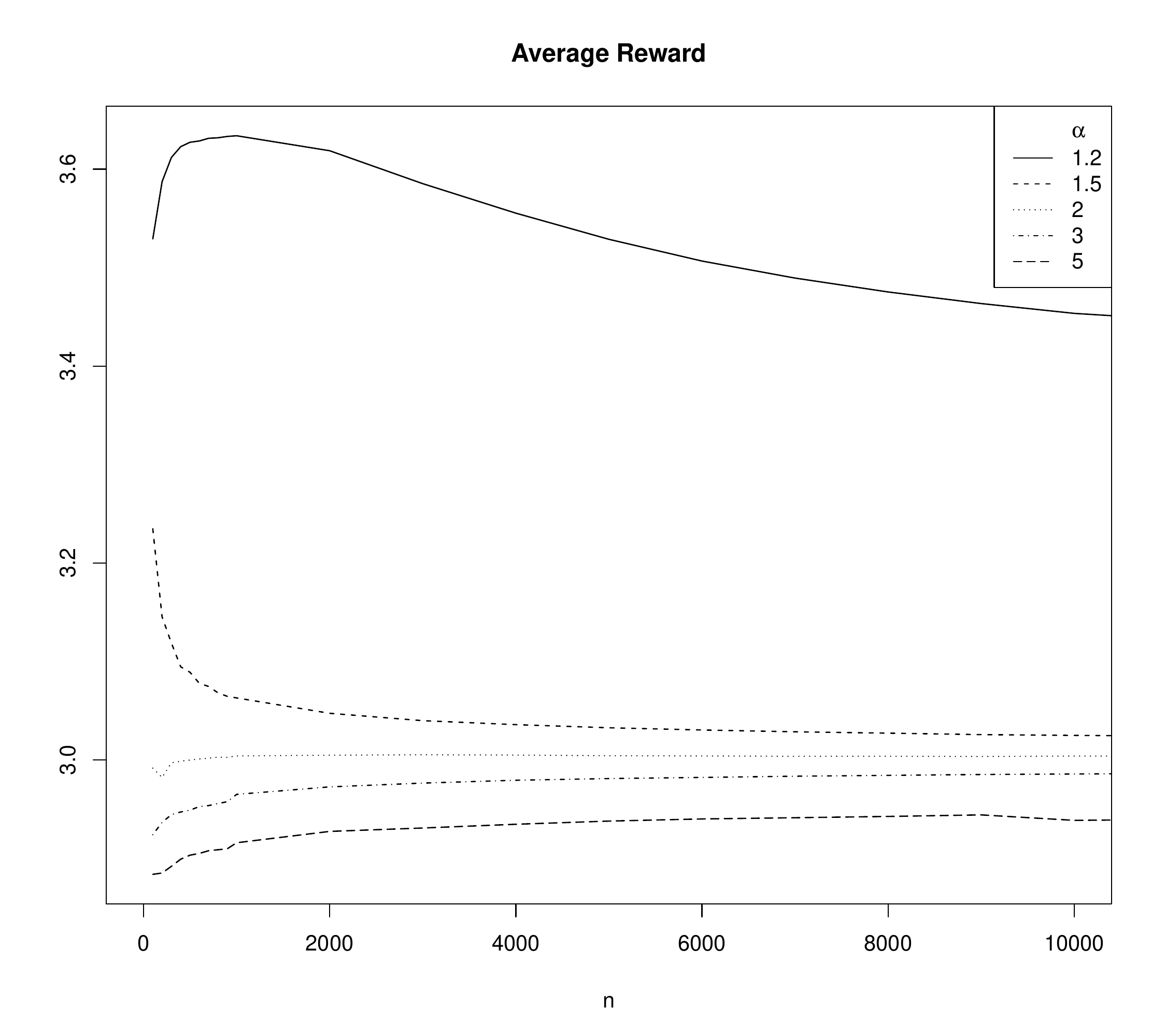}
  \caption{Comparison of Convergence Rates for Power Sparse Sequences}
  \label{fig:av_var}
\end{figure}

We observe in Figure \ref{fig:av_var} that the convergence is slower
both for small and large values of $b$ and faster for intermediate
values. Especially for $b=1.2$ the difference is relatively large even
after 10000 periods. This is explained as follows. For small values of
$b$ the forced selections are more frequent. Although this has the
desirable effect that the mean estimates for all populations become
accurate very soon, it also means that non-optimal populations are
also sampled frequently because of forced selections. As a result the
average outcome may deviate from the true optimal value for a longer
time period. On the other hand, for large values of $b$ the sequences
$\tau_j$ all become very sparse and thus the forced selections are
rare. In this case it takes a longer time for the estimates to
converge, and the linear programming problems may produce non-optimal
solutions for long intervals.
 
It follows from the above discussion that intermediate values of $b$
are generally preferable, since they offer a better balance of the two
effects, fast estimation of all mean values and avoiding non optimal
populations. This is also evident in the graph, where the value $b=2$
seems to be the best in terms of speed of convergence. 

To address the question of accuracy of the comparison of convergence
rates based on simulation, Figure \ref{fig:av_a2} presents a 95\%
confidence region for the average outcome curve corresponding to
$b=2$, based on 1000 simulated scenarios. The confidence region is
generally very narrow (note that the vertical axes have different
scale in the two figures), thus the estimate of the expected average
outcome is quite accurate. This is also the case for the other curves,
therefore the comparison of convergence rates is valid. Furthermore,
the length of the confidence interval becomes smaller for larger time
periods since, as expected, the convergence to the true value is
better for longer scenario durations.

\begin{figure}[h]
  \centering
  \includegraphics[width=1\textwidth]{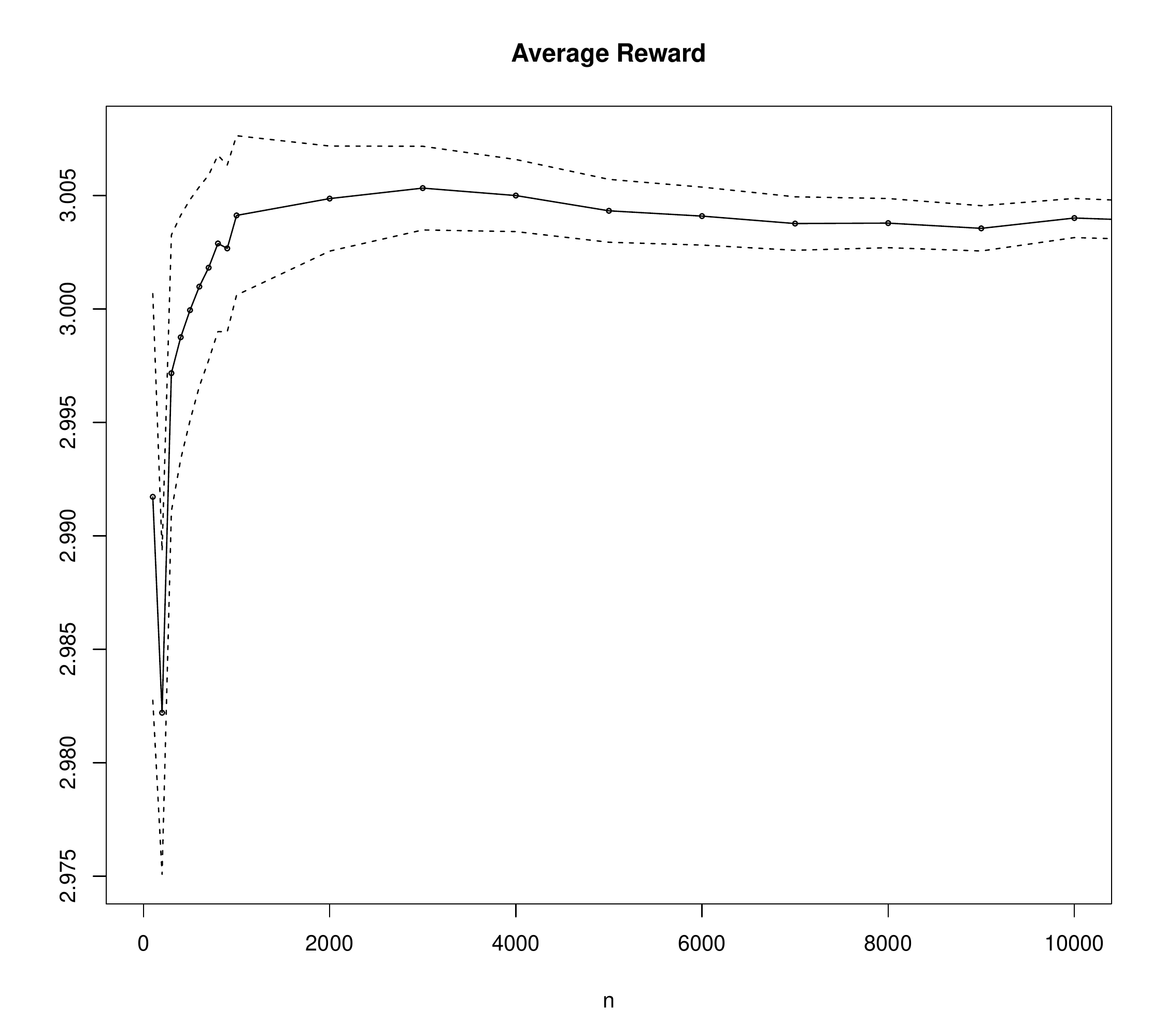}
  \caption{Confidence Region for Average Outcome for $b=2$}.
  \label{fig:av_a2}
\end{figure}

Another issue arising from Figure \ref{fig:av_var} is the
following. For $b=1.2$ the average outcome converges very slowly to
$z^*$, but remains above it for the entire scenario
duration. Thus it could be argued that, although the convergence is
not good, this policy is actually preferable, because it yields higher
average outcomes than the other policies. It also seems to contradict
the fact that $z^*$ is the maximum average outcome
under complete information, since there is a sampling policy that even
under incomplete information performs better. 

The reason for this discrepancy is related to the form of the cost
constraint \eqref{eq:feasible}.  The constraint requires the
infinite-horizon expected cost per period not to exceed $C_o$. This
does not preclude the possibility that one or more populations with
large sampling costs and large expected outcomes could be used for
arbitrarily long intervals before switching to a constrained-optimal
policy for the remaining infinite horizon. Such policies might achieve
average rewards higher than $z^*$ for long intervals, however this is
achieved by ``borrowing'', i.e., violating the cost constraint, also
for long time periods. Since \eqref{eq:feasible} is only required to
hold in the limit, this behavior of a policy is allowed.

Although the consistent policies in Section 3 are not designed
specifically to take advantage of this observation, they are neither
designed to avoid it. Therefore, it is possible, as it happens here for
$b=2$, that a consistent policy may achieve higher than optimal
average outcomes for long time periods before it converges to $z^*$.

The above discussion shows that the constraint as expressed in
\eqref{eq:feasible}, may not be appropriate, if for example the
sampling cost is a tangible amount that must be paid each time an
observation is taken, and there is a budget $C_0$ per period for
sampling. In this situation a policy may suggest exceeding the budget
for long time periods and still be feasible, something that may not be
viable in reality. In such cases it would be more realistic to impose
a stricter average cost constraint, for example to require that
\eqref{eq:feasible} hold for all $n$ and not only in the limit.

\section{Conclusion and Extensions}

In this paper we developed a family of consistent adaptive policies
for sequentially sampling from $k$ independent populations with
unknown distributions under an asymptotic average cost constraint. The
main idea in the development of this class of policies is to employ a
sparse sequence of forced selection periods for each population, to
ensure consistent estimation of all unknown means and in the remaining
time periods employ the solution obtained from a linear programming
problem that uses the estimates instead of the true values. We also
performed a simulation study to compare the convergence rate for
different policies in this class.

This work can be extended in several directions. First, as it was
shown in Section 4, the asymptotic form of the cost constraint is in
some sense weak, since it allows the average sampling cost to exceed
the upper bound for arbitrarily long time periods and still be
satisfied in the limit.  A more appropriate, albeit more complex,
model would be to require the cost constraint to be satisfied at all
time points. The construction of consistent and, more importantly,
efficient policies under this stricter version of the constraint is
work currently in progress.

Another extension is towards the direction of Markov process
control. Instead of assuming distinct independent populations with
i.i.d. observations, one might consider an average reward Markovian
Decision Process with unknown transition law and/or reward
distributions, and one or more nonasymptotic side constraints on the
average cost.  In this case the problem is to construct consistent
and, more importantly, efficient control policies, extending the
results of \cite{ad-mdp} in the constrained case.  

\centerline{\bf Acknowledgement}
This research was supported by the Greek Secreteriat of Research and
Technology under a Greece/Turkey bilateral research collaboration
program. The authors thank Nickos Papadatos and George Afendras for
useful discussions on the problem of consistent estimation in a random
sequence of random variables.

\bibliographystyle{abbrvnat}

\end{document}